\documentclass[11pt, article]{alt2023}
\usepackage[utf8]{inputenc}
\usepackage{amsmath}
\usepackage{amsfonts}
\usepackage{graphicx}
\usepackage{mathtools}
\usepackage{pst-all}
\usepackage{tikz}
\usepackage{pgfplots}
\usepackage[T1]{fontenc}
\graphicspath{ {./Images/} }
\usepackage{tikz}
\usepackage{comment}
\newcommand{\interior}[1]{%
  {\kern0pt#1}^{\mathrm{o}}%
}
\newcommand\restr[2]{{
  \left.\kern-\nulldelimiterspace 
  #1 
  \vphantom{\big|} 
  \right|_{#2} 
  }}

\usepackage[backend=biber,style=numeric,]{biblatex}

\DeclarePairedDelimiter\abs{\lvert}{\rvert}%
\DeclarePairedDelimiter\norm{\lVert}{\rVert}%

\DeclareMathOperator*{\argmin}{arg\,min}
\makeatletter
\let\oldabs\abs
\def\abs{\@ifstar{\oldabs}{\oldabs*}}
\let\oldnorm\norm
\def\norm{\@ifstar{\oldnorm}{\oldnorm*}}
\makeatother

\title[Optimal AdaBoost Converges]{Optimal AdaBoost Converges}
\usepackage{times}

\altauthor{%
 \Name{Conor Snedeker} \Email{snedeker.31@buckeyemail.osu.edu}\\
 \addr Ohio State University \\
 Computer Science \& Engineering Department 
}

\begin{document}
\maketitle

\begin{abstract}
The following work is a preprint collection of formal proofs regarding the convergence properties of the AdaBoost machine learning algorithm's classifier and margins. Various math and computer science papers have been written regarding conjectures and special cases of these convergence properties \cite{rudin}\cite{ortiz}\cite{marconv}\cite{ratsch}. Furthermore, the margins of AdaBoost feature prominently in the research surrounding the algorithm \cite{softmargin}\cite{rudin2007}\cite{ogmargin}\cite{marginvs}\cite{wang}\cite{reyzin}. At the zenith of this paper we present how AdaBoost's classifier and margins converge on a value that agrees with decades of research \cite{ratsch}\cite{rudin}\cite{boostingfound}. After this, we show how various quantities associated with the combined classifier converge.
\end{abstract}

\section{Introduction}

The \emph{margins hypothesis} with respect to the effectiveness of AdaBoost is the leading explanation for how the algorithm achieves good generalization on a wide range of data sets \cite{rudin}\cite{schap}. The hypothesis states that AdaBoost converges on a distribution of its decision margins on its training set that also improves its classification effectiveness over time. There has been much work on giving sufficient conditions for \emph{good} margin distributions \cite{marconv}\cite{marginvs}\cite{softmargin}\cite{ogmargin} along with conditions for minimum margin maximization \cite{rudin}\cite{marginvs}.

However, even given this research the distribution of these margins of AdaBoost is not well-understood. A key reason for that, we believe, is that the tools for analyzing the margins have been of a particular nature. Whereas much of the literature uses optimization and probabilistic tools, in this paper we present information theory and ergodic theory-inspired methods. Using these ideas we hope to find new inroads to analyzing AdaBoost and perhaps other algorithms.

\section{Preliminaries}
Suppose that $\mathcal{X}$ is a features space and labels $\mathcal{Y}=\{\pm 1\}$. Take some fixed training set $S\in(\mathcal{X}\times\mathcal{Y})^n$. Suppose that we have a set of hypotheses $\mathcal{H}\subseteq\{\pm 1\}^{\mathcal{X}}$. The set $\mathcal{X}\times\mathcal{Y}$ may be denoted $\mathcal{Z}=\mathcal{X}\times\mathcal{Y}$ as times and generally $z=(x,y)\in\mathcal{Z}$ is an arbitrary data point. We will be proving our results by treating AdaBoost as a function of deterministic variables. 

Let $P$ be a discrete probability distribution represented as a vector $P=\langle p_i\rangle^n_{i=1}$ such that $\sum^n_{i=1}p_i=1$. The \textbf{information content} of a probability value $p$ is the quantity $-\log p$. This quantity can be thought of as the information yielded by the value of $p$ over the distribution $P$. Generally, a random variable is used in the place of $p$, but the information content of that random variable is calculated using its corresponding probability value from a distribution. We define the \textbf{entropy} \cite{info0} of this distribution to be
\[H(P)=-\sum^n_{i=1}p_i\log p_i.\]

Given that we will be working on parameters of the AdaBoost algorithm, we will define some of these values. This algorithm can be thought of as an iterative update on the normalized \textbf{weight vector} $\vec{w}_t=\langle w_{t,i}\rangle^n_{i=1}$ for $t\in\mathbb{N}$ an iteration of AdaBoost. These weight vectors are initialized as $\vec{w}_0=\left\langle\frac{1}{n}\right\rangle^n_{i=1}$ and always sum to $1$.

AdaBoost updates the weight vector over many iterations, in doing so it requires a \textbf{mistake dichotomy} $\eta_{t}=\langle\eta_{t,i}\rangle^n_{i=1}$ with $\eta_{t,i}=\pm 1$. These are taken from a modified dichotomy set induced by the hypothesis space $\mathcal{H}$ denoted
\[\mathcal{C}_{\vec{y}}=\{\langle y_ih(x_i)\rangle^n_{i=1}:z_i=(x_i,y_i)\in S,\ h\in\mathcal{H}\}.\]
This mistake dichotomy is used to generate the \textbf{edge} at iteration $t$ called $r_t$ and defined by
\[r_t=\argmin_{\eta\in\mathcal{H}_{\vec{y}}}(\vec{w}_t\cdot\eta)\]
where $\eta\cdot\vec{w}_t$ is the conventional vector dot product. Note that we requires $r_t>0$ for all $t$ since $r_t=1-2\epsilon_t$ with $\epsilon_t$ the error of our hypothesis $h_t$. This is known as the \emph{weak learning condition} and ensures that $h_t$ is better than random guessing. For our use, we also require that $r_t<1$ since $r_t=1$ corresponds to trivial dynamics \cite{rudin} and means that we have found a hypothesis that determines all labels of the training set. 

The edge value is used to define the \textbf{learning coefficient} $\alpha_t$, or simply \textbf{coefficient}, defined
\[\alpha_t=\frac{1}{2}\log\left(\frac{1+r_t}{1-r_t}\right)\]
which ends up being used to weight the \textbf{combined classifier} $F_t(x)$ as
\[F_t(x)=\sum^t_{k=0}\alpha_k h_k(x)\]
for $h_k\in\mathcal{H}$.

The learning coefficient is used in the AdaBoost \textbf{weight update}
\[w_{t+1,i}=\frac{w_{t,i}e^{-\eta_{t,i}\alpha_t}}{Z_t}\]
where $Z_t=\sum^n_{i=1}w_{t,i}e^{-\eta_{t,i}\alpha_t}$, which we call the \textbf{partition function} at iteration $t$. It is a classic result that $Z_t=\sqrt{1-r^2_t}$ \cite{schap}.

Our final AdaBoost parameter is the \textbf{margin} of the $i$-th data point at iteration $t$ which is given by
\[y_iF_t(x_i)=\sum^t_{k=0}\eta_{t,i}\alpha_k\]
and this tracks the \emph{confidence} of the final classifier with respect to its classification on the $i$-th data point. We will often represent the margin of data point $i$ at iteration $t$ as
\[\text{mar}_{t,i}=y_iF_t(x_i)\]
in order to make things look nice and intuitive.
\begin{algorithm}
 \caption{\emph{Optimal} AdaBoost \cite{me}}\label{alg:adaboost}
 \SetAlgoLined
 \KwData{$\mathcal{S},t_{max}$}
 \KwResult{Combined classifier}
 initialization\;
 
 $F_0(x)=0$
 
 $w_{i,0} \gets \frac{1}{n}$ \text{\ \ \ \ \ \ \ \ \ \ \ \ \ \ \ \ \ \ \ \ \ \ \ \ \ \ \ \ \ \ \ \ \ \ \ \ \ \ \ \ \ \ \ \ \ \ for $n$ components of $\vec{w}_0$}
 
 $t \gets 0$\;
 
 \For{$t\leq t_{max}$}{
 
  $\eta_t\in\text{argmax}_{\eta'\in\mathcal{C}_{\vec{y}}}(\vec{w}_t\cdot\eta')$\;
  
  $r_t=\vec{w}_t\cdot\eta_t$ \text{\ \ \ \ \ \ \ \ \ \ \ \ \ \ \ \ \ \ \ \ \ \ \ \ \ \ \ \ \ \ \ \ \ \ \ \ the optimal edge at iter. $t$}\;
  
  $\alpha_t=\frac{1}{2}\log\left(\frac{1+r_t}{1-r_t}\right)$\;
  
  $F_t(x)=F_{t-1}(x)+\alpha_th_{t}(x)$ \text{\ \ \ \ \ \ \ \ \ \ update combined classifier}\;
  
  $w_{t,i}\gets w_{t,i}e^{-\eta_{t,i}\alpha_t}$\;
  
  $w_{t+1,i}=\frac{w_{t,i}}{Z_{t}}$ \text{\ \ \ \ \ \ \ \ \ \ \ \ \ \ \ \ \ \ \ \ \ \ \ \ \ \ \ \ \ \  \ \ \        normalization for each $i$}\;
  
}
\textbf{return} $F_{t_{max}}(x)$ \text{\ \ \ \ \ \ \ \ \ \ \ \ \ \ \ \ \ \ \ \ \ \ \ \ \ \ \ \ \ \ \ \ \ \ final classifier}

\end{algorithm}

We can take the expected value of the $n$-many valued random variables $X\in\{x_1,...,x_n\}$ with respect to the distribution $\vec{w}_{t}$ in the manner of
\[\mathbb{E}_{\vec{w}_{t}}[X]=\sum^n_{i=1}w_{t,i}x_i.\]
In particular we can calculate the expected value of the sum of margins at iteration $t$ named $\text{mar}_t=\sum^n_{i=1}\text{mar}_{t,i}$ via
\[\mathbb{E}_{\vec{w}_{t+1}}[\text{mar}_t]=\sum^n_{i=1}w_{t+1,i}y_iF_t(x_i).\]
This expected value is over the discrete distribution that $\vec{w}_{t+1}$ defines rather than the distribution of the underlying data from which we take $S$. The expected value in this case resembles the \emph{mean energy} of the Ising model \cite{phys}, which AdaBoost in turn greatly resembles.

\section{Formal Results}

A nice evaluation can be found in simply applying $-\log$ to $w_{t+1,i}$ given the iterative weight update formula of AdaBoost. This quantity is
\begin{equation}\label{eqn:info_content}
    -\log w_{t+1,i}=-\log\left(\frac{1}{n}\prod^t_{k=0}\frac{e^{-\eta_{l,i}\alpha_k}}{Z_k}\right)=\log n+\sum^t_{k=0}\eta_{l,i}\alpha_k+\sum^t_{k=0}\log Z_k.
\end{equation}
The above fact and its consequences will be used repeatedly throughout this work.

\begin{proposition}\label{prop:kl_bounds}
Suppose that AdaBoost is at iteration $\vec{w}_{t+1}$ and suppose that $k\in\mathbb{Z}$ with $0\leq k\leq t$. Then we have
\[-\log n-\sum^t_{k=0}\frac{1}{2}\log\left(1-r^2_k\right)\leq\mathbb{E}_{\vec{w}_{t+1}}[\emph{\text{mar}}_t]\leq-\sum^t_{k=0}\frac{1}{2}\log\left(1-r^2_k\right).\]
\end{proposition}
\begin{proof}
Applying the entropy function $H$ to $\vec{w}_{t+1}$ and using the identity shown before this proposition, we have that
\begin{equation*}
\begin{aligned}
    H(\vec{w}_{t+1})&=-\sum^n_{i=1}w_{t+1,i}\left(-\log n-\sum^t_{k=0}\eta_{k,i}\alpha_k-\sum^t_{k=0}\log Z_k\right)\\
    &=\log n+\sum^n_{i=1}w_{t+1,i}\sum^t_{k=0}\eta_{k,i}\alpha_k+\sum^t_{k=0}\log Z_k\\
    &=\log n+\sum^n_{i=1}w_{t+1,i}y_iF_t(x_i)+\sum^t_{k=0}\log Z_k
\end{aligned}
\end{equation*}
and this leads us to
\begin{equation*}
\begin{aligned}
    H(\vec{w}_{t+1})-\log n-\sum^t_{k=0}\log Z_k&=\sum^n_{i=1}w_{t+1,i}y_iF_t(x_i)\\
    &=\mathbb{E}_{\vec{w}_{t+1}}[\text{mar}_t].
\end{aligned}
\end{equation*}
Since $0\leq H(P)$ for any distribution $P$, we get the inequality
\[-\log n-\sum^t_{k=0}\log Z_k\leq\mathbb{E}_{\vec{w}_{t+1}}[\text{mar}_t].\]
So
\[-\log n-\sum^t_{k=0}\frac{1}{2}\log\left(1-r^2_k\right)\leq\mathbb{E}_{\vec{w}_{t+1}}[\text{mar}_t].\]
Similarly, since $H(P)\leq \log n$
\[\mathbb{E}_{\vec{w}_{t+1}}[\text{mar}_t]\leq \log n-\log n-\sum^t_{k=0}\frac{1}{2}\log\left(1-r^2_k\right)\]
and hence the result.
\end{proof}
In the above sense, we see that AdaBoost maintains a growth relation between the sum $-\sum^t_{k=0}\frac{1}{2}\log\left(1-r^2_k\right)$ and our expected value of margins. We will see later that this has to do with the minimum margin in specific.

\begin{definition}[\cite{rudin}]
We call an unlabeled training example $x_i$ for $i\in[n]$ a \textbf{support vector} if there exists $t_0\in\mathbb{N}\cup\{0\}$ so that $\overline{\text{mar}}_{t,i}$ achieves and maintains the minimum margin over training examples as $t$ grows large. Further, let $T$ be the set of training examples without their labels. Define the set
\[V=\{x_i\in T:x_i \emph{\text{ is a support vector}}\}\]
and call $V$ the \textbf{set of support vectors} with respect to $T$.
\end{definition}
Any training example that achieves and maintains the minimum margin as $t$ grows large will have a respective weight $w_{t,i}$ that stays positive, where as for $x_j\in T\setminus V$ will have $w_{t,j}\rightarrow 0$ or else $w_{t,j}$ oscillates between $0$ and some positive values.
\begin{definition}
Let $i\in[n]$. Then we define the \textbf{margin with normalized coefficients} or \textbf{normalized margin} to be
\[\overline{\emph{\text{mar}}}_{t,i}=\frac{1}{\sum^t_{k=0}\alpha_k}\emph{\text{mar}}_{t,i}\]
and we call $\sum^t_{k=0}\alpha_k$ the \textbf{normalization constant} where
\[A_t=\sum^t_{k=0}\alpha_k.\]
\end{definition}
The normalization constant $A_t$ normalizes the learning coefficients of the margin so that they sum to $1$.
\begin{definition}
Suppose that AdaBoost is at iteration $t$ and we have the combined classifier $F_t(x)$ constructed using Algorithm~\ref{alg:adaboost}. We call the quantity
\[f_t(x)=\frac{F_t(x)}{A_t}\]
the \textbf{normalized classifier} of Optimal AdaBoost.
\end{definition}
\begin{proposition}\label{prop:up_low_bd}
Suppose that $i\in I(V)$. Let $\epsilon >0$ be a constant so that $\epsilon<w_{t+1,i}\leq 1$ for all $t$ sufficiently large by definition of support vector. Then
\[\emph{\text{mar}}_{t,i}-\sum^t_{k=0}\frac{1}{2}\log\left(1-r^2_k\right)\]
is bounded above and below by finite constants.
\end{proposition}
\begin{proof}
By Eqn.~\ref{eqn:info_content} we have that
\[-\log w_{t+1,i}=\log n+\text{mar}_{t,i}-\sum^t_{k=0}\frac{1}{2}\log\left(1-r^2_k\right).\]
for each $i\in[n]$. Then since $\epsilon<w_{t+1,i}$, applying $-\log$ to both sides of the inequality gives us $-\log w_{t+1,i}<-\log\epsilon$ for all $t$. This means
\[0\leq\log n+\text{mar}_{t,i}-\sum^t_{k=0}\frac{1}{2}\log\left(1-r^2_k\right)<-\log\epsilon\]
such that 
\[-\log n\leq\text{mar}_{t,i}-\sum^t_{k=0}\frac{1}{2}\log\left(1-r^2_k\right)<-\log n\epsilon\] 
for all $t$.
\end{proof}
\begin{lemma}
Let $i\in I(V)$. Then
\[\lim_{t\rightarrow\infty}\overline{\emph{\text{mar}}}_{t,i}=\lim_{t\rightarrow\infty}-A^{-1}_t\sum^t_{k=0}\frac{1}{2}\log\left(1-r^2_k\right).\]
Furthermore, the rate convergence depends only on $\abs{S}=n$ and the distribution of values $(r_k)^{\infty}_{k=0}$.
\end{lemma}
\begin{proof}
By Proposition~\ref{prop:up_low_bd} we have that
\[-\log n\leq\text{mar}_{t,i}-\sum^t_{k=0}\frac{1}{2}\log\left(1-r^2_k\right)<-\log n\epsilon.\]
Multiplying all parts of these inequalities by $A^{-1}_t$ gives us
\[-A^{-1}_t\log n\leq\overline{\text{mar}}_{t,i}-A^{-1}_t\sum^t_{k=0}\frac{1}{2}\log\left(1-r^2_k\right)<-A^{-1}_t\log n\epsilon.\]
Since $-\log n$ and $\epsilon$ are constants, taking the limit $t\rightarrow\infty$ proves the lemma.
\end{proof}
Observe that another way to write the value
\[-A^{-1}_t\sum^t_{k=0}\frac{1}{2}\log\left(1-r^2_k\right)\]
is via
\[\frac{-\sum^t_{k=0}\log(1-r^2_k)}{\sum^t_{k=0}\log\left(\frac{1+r_k}{1-r_k}\right)}.\]
A similar function over a finite index appears in a 2004 paper by Rudin, Daubechies, and Schapire \cite{rudin} for use in the cycling dynamics of AdaBoost. It also appears as a single-term version without sums in other literature such as in \emph{Boosting: Foundations and Algorithms} chapter 5 \cite{boostingfound}. The single-term version was introduced by R{\"a}tsch and Warmuth in 2005 \cite{ratsch}. When written with a single term in numerator and denominator it describes a game theoretic relationship between the edge and minimum margin.

For our purposes now and with the formalism that we have built up over the course of this work, we can write the asymptotic support vector margin as
\[\overline{\text{mar}}_{\infty,i}=\frac{-\sum^\infty_{k=0}\log(1-r^2_k)}{\sum^\infty_{k=0}\log\left(\frac{1+r_k}{1-r_k}\right)}\]
where $\infty$ in place of $t$ denotes an infinite limit. Since all support vectors have this same limit, they are asymptotically identical. Given that we cannot bound non-support vectors as we did in Proposition~\ref{prop:up_low_bd}, it is not clear if they have such an asymptotic identity. Nothing too mysterious is going on when taking the limit in this case as all margins with normalized coefficients are in $[0,1]$, which means their limit is too. One may wonder about oscillation, which we deal with in a coming lemma.
\begin{proposition}\label{prop:expect_kl}
Suppose that AdaBoost is at iteration $\vec{w}_{t+1}$. Then
\[\lim_{t\rightarrow\infty}\mathbb{E}_{\vec{w}_{t+1}}[\overline{\emph{\text{mar}}}_t]=\lim_{t\rightarrow\infty}-A^{-1}_t\sum^t_{k=0}\frac{1}{2}\log\left(1-r^2_k\right).\]
As in the previous lemma, the rate of convergence will depend only on $\abs{S}=n$ and the distribution of values $(r_k)^\infty_{k=0}$.
\end{proposition}
\begin{proof}
We know that
\[\mathbb{E}_{\vec{w}_{t+1}}[\text{mar}_t]=\sum^n_{i=1}w_{t+1,i}\text{mar}_{t,i}\]
and also
\[-\log n-\sum^t_{k=0}\frac{1}{2}\log\left(1-r^2_k\right)\leq\mathbb{E}_{\vec{w}_{t+1}}[\text{mar}_t]\leq -\sum^t_{k=0}\frac{1}{2}\log\left(1-r^2_k\right).\]
from Proposition~\ref{prop:kl_bounds}. Multiplying through by $A^{-1}_t$ on the second equality gives us
\[A^{-1}_t\left(-\log n-\sum^t_{k=0}\frac{1}{2}\log\left(1-r^2_k\right)\right)\leq A^{-1}_t\mathbb{E}_{\vec{w}_{t+1}}[\text{mar}_t]\leq A^{-1}_t\left(-\sum^t_{k=0}\frac{1}{2}\log\left(1-r^2_k\right)\right).\]
When we take the limit $t\rightarrow\infty$ the proposition follows.
\end{proof}
What is important to note for this proposition and the previous lemma is that although we are taking limits, these processes also converge as $t$ grows large but finite. While we have chosen to use limits due to their analytical beauty, we could forego this in respecting the context of finite time in computer science applications. In this sense, we are also giving finite bounds on both the expected value of margins along with the individual values of support vector margins. That these things apply in the most general of cases where we have not specified $S$ nor $\mathcal{H}$ is quite amazing.

The following proofs resemble the results given in a paper from 2015 by Joshua Belanich and Luis Ortiz that conjectured AdaBoost as a measure-preserving dynamical system \cite{ortiz}. We originally sought to prove the conjecture, but the convergence properties of the algorithm follow without any such measure theoretic properties.

\begin{lemma}\label{lem:converges}
The limit of normalized margins converges to a constant value.
\end{lemma}
\begin{proof}
Fix $l\in\mathbb{N}$ and suppose that $i\in I(V)$. Now, consider the difference
\[\overline{\text{mar}}_{t,i}-\overline{\text{mar}}_{t+l,i}.\]
We will prove the lemma by showing that the above quantity equals $0$ as $t\rightarrow\infty$. This will mean that the limiting value of the margin does not oscillate indefinitely. Since the normalized margins are bounded, this implies convergence. Now
\begin{equation*}
\begin{aligned}
\overline{\text{mar}}_{t,i}-\overline{\text{mar}}_{t+l,i}&=\frac{\text{mar}_{t,i}}{A_{t}}-\frac{\text{mar}_{t+l,i}}{A_{t+l}}\\
&=\frac{\text{mar}_{t,i}A_{t+l}}{A_tA_{t+l}}-\frac{\text{mar}_{t+l,i}A_t}{A_tA_{t+l}}.
\end{aligned}
\end{equation*}
Then we turn our attention to the difference in the numerator such that
\begin{equation*}
\begin{aligned}
\text{mar}_{t,i}A_{t+l}-\text{mar}_{t+l,i}A_{t}&=\text{mar}_{t,i}\sum^{t+l}_{k=0}\alpha_k-\text{mar}_{t+l,i}\sum^{t}_{k=0}\alpha_k \\
&=\text{mar}_{t,i}\sum^{t}_{k=0}\alpha_k+\text{mar}_{t,i}\sum^{t+l}_{k=t+1}\alpha_k-\text{mar}_{t,i}\sum^{t}_{k=0}\alpha_k-\sum^{t+l}_{k=t+1}\eta_{k,i}\alpha_k\sum^{t}_{k=0}\alpha_k\\
&=\text{mar}_{t,i}\sum^{t+l}_{k=t+1}\alpha_k-\sum^{t+l}_{k=t+1}\eta_{k,i}\alpha_k\sum^{t}_{k=0}\alpha_k.
\end{aligned}
\end{equation*}
So
\begin{equation*}
\begin{aligned}
\overline{\text{mar}}_{t,i}-\overline{\text{mar}}_{t+l,i}&=\frac{\text{mar}_{t,i}\sum^{t+l}_{k=t+1}\alpha_k}{A_tA_{t+l}}-\frac{\sum^{t+l}_{k=t+1}\eta_{k,i}\alpha_k\sum^{t}_{k=0}\alpha_k}{A_tA_{t+l}}\\
&=\frac{\overline{\text{mar}}_{t,i}\sum^{t+l}_{k=t+1}\alpha_k}{A_{t+l}}-\frac{\sum^{t+l}_{k=t+1}\eta_{k,i}\alpha_k}{A_{t+l}}\\
\end{aligned}
\end{equation*}
Observe that both terms
\[\overline{\text{mar}}_{t,i}\sum^{t+l}_{k=t+1}\alpha_k,\ \sum^{t+l}_{k=t+1}\eta_{k,i}\alpha_k\]
are bounded above by $\sum^{t+l}_{k=t+1}\alpha_k$, a finite quantity for all $t$ since $0<r_k<1$ for $0\leq k\leq t+l$. Hence
\[\lim_{t\rightarrow\infty}\left(\overline{\text{mar}}_{t,i}-\overline{\text{mar}}_{t+l,i}\right)=\lim_{t\rightarrow\infty}\left(\frac{\overline{\text{mar}}_{t,i}\sum^{t+l}_{k=t+1}\alpha_k}{A_{t+l}}-\frac{\sum^{t+l}_{k=t+1}\eta_{k,i}\alpha_k}{A_{t+l}}\right)=0\]
given that $A_{t+l}\rightarrow\infty$ when $t\rightarrow\infty$. This completes the proof since $l$ was arbitrarily chosen.
\end{proof}
\begin{corollary}
The limit for the expected value of the normalized margins converges to a constant value.
\end{corollary}
\begin{proof}
The proof for this follows from the above lemma since the limit value of any support vector of a normalized margin is the same as the limit of the expected value of the normalized margins.
\end{proof}
\begin{definition}
Consider the set of labelling dichotomies induced by our hypotheses $\mathcal{H}$ on $T$ the unlabeled training set
\[\mathcal{C}=\{\langle h(x_i)\rangle^n_{i=1}:x_i\in T,\ h\in\mathcal{H}\}.\]
Let $j\in[\abs{\mathcal{C}}]$ and suppose that $K_{j,t}$ indexes the iterations up to $t$ at which AdaBoost selects a hypothesis with dichotomy $\mu_j\in\mathcal{C}$. We can identify to each $\mu_j$ the normalized coefficients that will multiply them in the final classifer up to iteration $t$ with
\[\lambda_{t,j}=\frac{\sum_{k_j\in K_{j,t}}\alpha_{k_j}}{A_t}.\]
\end{definition}
\begin{definition}
Let $i\in[n]$ with AdaBoost at iteration $t$ and consider $\overline{\text{mar}}_{t,i}$. Define the index $N^+_{t,i}$ to be iterations $k$ up to $t$ so that $\eta_{k,i}=+1$ and $N^-_{t,i}$ the same for $\eta_{k,i}=-1$. We define the value $\beta^{\pm}_{t,i}$ to be
\[\beta^{\pm}_{t,i}=\sum_{n^{\pm}_i\in N^{\pm}_{t,i}}\alpha_{n^{\pm}_i}.\]
\end{definition}
Observe that
\[\pm\beta^{\pm}_{t,i}=\overline{\text{mar}}_{t,i}\pm\beta^{\mp}_{t,i}.\]
This quantity defines the total contribution of $\eta_{k,i}=\pm 1$ for each $k$ to the classification of a data point.
\begin{proposition}
For each dichotomy $\mu_j\in\mathcal{C}$ the value $\lim_{t\rightarrow\infty}\lambda_{t,j}$ converges.
\end{proposition}
\begin{proof}
Like the lemma above we take $l\in\mathbb{N}$ fixed and consider
\[\lambda_{t,j}-\lambda_{t+l,j}.\]
This value is
\begin{equation*}
\begin{aligned}
\lambda_{t,j}-\lambda_{t+l,j}&=\frac{\sum_{k_j\in K_{j,t}}\alpha_{k_j}}{A_t}-\frac{\sum_{k_j\in K_{j,t+l}}\alpha_{k_j}}{A_{t+l}}\\
&=\frac{A_{t+l}\sum_{k_j\in K_{j,t}}\alpha_{k_j}}{A_tA_{t+l}}-\frac{A_t\sum_{k_j\in K_{j,t+l}}\alpha_{k_j}}{A_tA_{t+l}}.\\
\end{aligned}
\end{equation*}
Then
\begin{equation*}
\begin{aligned}
A_{t+l}\sum_{k_j\in K_{j,t}}\alpha_{k_j}-A_t\sum_{k_j\in K_{j,t+l}}\alpha_{k_j}&=A_{t}\sum_{k_j\in K_{j,t}}\alpha_{k_j}+\left(\sum^{t+l}_{k=t+1}\alpha_k\right)\sum_{k_j\in K_{j,t}}\alpha_{k_j}\\
&-A_t\sum_{k_j\in K_{j,t}}\alpha_{k_j}-A_t\sum_{k_j\in K_{j,t+l}\symbol{92}K_{j,t}}\alpha_{k_j}\\
&=\left(\sum^{t+l}_{k=t+1}\alpha_k\right)\sum_{k_j\in K_{j,t}}\alpha_{k_j}-A_t\sum_{k_j\in K_{j,t+l}\symbol{92}K_{j,t}}\alpha_{k_j}.
\end{aligned}
\end{equation*}
The above implies that
\begin{equation*}
\begin{aligned}
\lambda_{t,j}-\lambda_{t+l,j}&=\frac{\left(\sum^{t+l}_{k=t+1}\alpha_k\right)\sum_{k_j\in K_{j,t}}\alpha_{k_j}}{A_tA_{t+l}}-\frac{A_t\sum_{k_j\in K_{j,t+l}\symbol{92}K_{j,t}}\alpha_{k_j}}{A_tA_{t+l}}\\
&=\frac{\left(\sum^{t+l}_{k=t+1}\alpha_k\right)\lambda_{t,j}}{A_{t+l}}-\frac{\sum_{k_j\in K_{j,t+l}\symbol{92}K_{j,t}}\alpha_{k_j}}{A_{t+l}}.\\
\end{aligned}
\end{equation*}
As in the previous lemma, both terms have bounded numerators in the difference above. This means that
\[\lim_{t\rightarrow\infty}(\lambda_{t,j}-\lambda_{t+l,j})=0\]
completing the proof.
\end{proof}

\begin{proposition}
The term $\pm\beta^{\pm}_{t,i}$ converges as $t\rightarrow\infty$ for all $i\in[n]$.
\end{proposition}
\begin{proof}
This result follows from a proof exactly like that for the above proposition.
\end{proof}
\begin{proposition}
The set $V$ of support vectors is non-empty.
\end{proposition}
\begin{proof}
Since the normalized margins converge, there must be a minimum normalized margin in the limit. Any finite set of real numbers has a minimum. This means that for some $i\in[n]$ and fixed iteration $t_0$, for all iterations $t$ so that $t_0\leq t$ the value $\overline{\text{mar}}_{t,i}$ attains the minimum margin value and stays there. Hence, $x_i\in V$ as $t\rightarrow\infty$.
\end{proof}
\begin{theorem}\label{thm:sup_vecs}
Only support vectors contribute to the value of $\mathbb{E}_{\vec{w}_{t}}[\overline{\emph{\text{mar}}}_{t}]$ as $t\rightarrow\infty$. Furthermore, we have that $\abs{V}>1$.
\end{theorem}
\begin{proof}
Suppose that 
\[\theta_t=\min_{j\in[n]}\overline{\text{mar}}_{t,j}.\]
We know
\[\lim_{t\rightarrow\infty}\mathbb{E}_{\vec{w}_{t+1}}[\overline{\text{mar}}_t]=\lim_{t\rightarrow\infty}-A^{-1}_t\sum^t_{k=0}\frac{1}{2}\log\left(1-r^2_k\right)\]
and since $V\neq\emptyset$ there is $i\in[n]$ with $x_i\in V$ so that
\[\lim_{t\rightarrow\infty}\overline{\text{mar}}_{t,i}=\lim_{t\rightarrow\infty}-A^{-1}_t\sum^t_{k=0}\frac{1}{2}\log\left(1-r^2_k\right).\]
Given the above we must have that there exists fixed iteration $t_0$ so that for all $t$ with $t_0\leq t$ the equality $\theta_t=\overline{\text{mar}}_{t,i}$ holds, i.e. $x_i$ has the least margin for large enough $t$. The first limit also means that for any $\epsilon >0$ there is large enough $t$ which gives
\[\abs{\sum^n_{j=1}w_{t,j}\overline{\text{mar}}_{t,j}-\theta_t}<\epsilon.\]
Now, rewriting $\theta_t$ to be a weighted sum over the sole term $\theta_t$ gives
\begin{equation*}
\begin{aligned}
\abs{\sum^n_{j=1}w_{t,j}\overline{\text{mar}}_{t,j}-\sum^n_{j=1}w_{t,j}\theta_t}&=\abs{\sum^n_{j=1}w_{t,i}\left(\overline{\text{mar}}_{t,j}-\theta_t\right)}<\epsilon\\
\end{aligned}
\end{equation*}
which can only be the case if $\overline{\text{mar}}_{t,j}-\theta_t\rightarrow 0$ or $w_{t,j}\rightarrow 0$ for each $j\in[n]$ as $t\rightarrow\infty$. Since AdaBoost cannot converge on a fixed weight vector $\vec{w}$ with only one non-zero term by the weak learning condition, there must be more than one support vector in the limit.
\end{proof}
\begin{theorem}\label{thm:AB_conv}
The normalized classifier that AdaBoost outputs $f_t(x)$ converges asymptotically.
\end{theorem}
\begin{proof}
By Lemma~\ref{lem:converges} all of the normalized margins of AdaBoost converge. Since the normalized margins of Optimal AdaBoost are the same as its normalized classifier applied to individual training examples and multiplied by a constant, the normalized classifier converges as well.
\end{proof}
This concludes the formal proofs of this paper.

\section{Discussion}

Theorem~\ref{thm:AB_conv} comes from some interesting ways of dealing with the weight vector $\vec{w}_t$ in relation to various quantities of information theory. Our initial quantity of Eqn.~\ref{eqn:info_content} is like a fingerprint for AdaBoost up to the latest iteration $t$. All information about the run of the algorithm over the training set $S$ can be seen in this equation. The cardinality of $S$, combined loss at each iteration, and the margins of iteration $t-1$ can all be found therein. What is most interesting about the information content of $\vec{w}_{t,i}$ is that the vector itself is rather opaque to analysis as it is. However, a simple application of $-\log$ garners much in terms of the ultimate convergence properties of the algorithm as $t\rightarrow\infty$.

As well, given Theorem~\ref{thm:sup_vecs} there must be more than one support vector. Using this definition that primarily saw use in the cycling dynamics of AdaBoost \cite{rudin}, we can see that the algorithm converges on a specific distribution of smallest margins. It is possible to control these minimum margin values to show that, in some respect, certain data points will be attracted to a sort of learning limit set $V$. What is most interesting here is that a training example either attains the minimum margin and stays \emph{relevant} via $\vec{w}_{t,i}$ bounded away from zero, or else becomes dynamically irrelevant with respect to the effects of the weight vector.

A paper from 2020 by Keifeng Lyu and Jian Li \cite{lyu2019gradient} on homogeneous neural networks regards the normalized margins of these very different classifiers in a similar way. Although they do not relate the margins and normalized margins to information theoretic quantities as we do in this work, they are able to show results using approximations of margins whose error is bounded in a similar fashion to our own Proposition~\ref{prop:kl_bounds}. Indeed, as in Proposition~\ref{prop:expect_kl}, the divergence of the magnitude of a parameter used in the learning process causes their approximation to converge to the normalized margin being approximated. Bounding techniques of this kind seem important in understanding the convergence of certain algorithms. Further, we believe that the information content of normalized quantities, the vector $\vec{w}_{t}$ in our case, may reveal similar \emph{fingerprints} in the analysis of learning algorithms separate from AdaBoost.

\clearpage
\printbibliography[heading=bibintoc,title={References}]

\end{document}